\documentclass[runningheads]{llncs}
\usepackage{amsmath}
\usepackage{graphicx}
\usepackage{multirow}
\usepackage{diagbox}
\usepackage{enumitem}
\usepackage{float}
\usepackage{amsmath,amssymb}

\begin{document}
\title{Merging variables: one technique of search in pseudo-Boolean optimization\thanks{This is a version of the paper accepted to MOTOR 2019 conference (http://motor2019.uran.ru/). In this version we fixed a minor number of typos and presented more detailed proof of Lemma 4.}}

\titlerunning{Merging Variables Principle}

\author{Alexander A. Semenov\orcidID{0000-0001-6172-4801}}
\authorrunning{A. Semenov}
\institute{ISDCT SB RAS, Irkutsk, Russia\\
\email{biclop.rambler@yandex.ru}}

\maketitle           

\begin{abstract}
In the present paper we describe new heuristic technique, which can be applied to the optimization of pseudo-Boolean functions including Black-Box functions. This technique is based on a simple procedure which consists in transition from the optimization problem over Boolean hypercube to the optimization problem of auxiliary function in a specially constructed metric space. It is shown that there is a natural connection between the points of the original Boolean hypercube and points from the new metric space. For the Boolean hypercube with fixed dimension it is possible to construct a number of such metric spaces. The proposed technique can be considered as a special case of Variable Neighborhood Search, which is focused on pseudo-Boolean optimization.
Preliminary computational results show high efficiency of the proposed technique on some reasonably hard problems.
Also it is shown that the described technique in combination with the well-known (1+1)-Evolutionary Algorithm allows to decrease the upper bound on the runtime of this algorithm for arbitrary pseudo-Boolean functions.


\end{abstract}

\section{Basic notions and methods}

Let $\{0,1\}^n$ be a set of all possible binary vectors (strings) of length $n$. The set $\{0,1\}^n$ is sometimes called a \textit{Boolean hypercube}. Let us associate with $\{0,1\}^n$ a set consisting of $n$ symbols: $X = \{x_1, \ldots, x_n\}$. The elements of $X$ will be referred to as \textit{Boolean variables}. 
Further we will consider $\{0,1\}^n$ as a set of all possible assignments of variables from $X$. For an arbitrary $X' \subseteq X$ by $\{0,1\}^{|X'|}$ we will denote a set of all possible assignments of variables from $X'$.

A pseudo-Boolean function (PBF) \cite{Boros} is an arbitrary total function of~the~kind 
\begin{equation}
\label{eq1}
f:\{0,1\}^n\rightarrow \mathbb{R}.
\end{equation}

\begin{example}
Consider an arbitrary Conjunctive Normal Form (CNF) $C$, where $X = \{x_1, \ldots, x_n\}$ is a set of Boolean variables from this CNF. Let us associate with an arbitrary $\alpha \in \{0,1\}^n$ the number of clauses that take the value of $1$ when their variables take the values from $\alpha$. Denote the resulting function by $f_C$. It is easy to see that $f_C$ is a function of the kind $f_C:\{0,1\}^n \rightarrow \mathbb{N}_0$ ($\mathbb{N}_0 = \{0,1,2, \ldots\}$) and $\max\limits_{\{0,1\}^n}f_C \leq m$, where $m$ is the number of clauses in $C$. Then CNF $C$ is satisfiable if and only if  $\max\limits_{\{0,1\}^n}f_C = m$. 
The problem $f_C\underset{\{0,1\}^n}{\rightarrow}\max$ 
represents the 
optimization formulation of the Boolean Satisfiability problem (SAT) and is often referred to as MaxSAT \cite{Handbook}. This problem is NP-hard, so there is a huge class of combinatorial problems, which can be effectively reduced to it.
\end{example}

The main result of the present paper is a technique applicable in the context of several common metaheuristic schemes. Before proceeding to its description, let us briefly describe the basic metaheuristics used below.

First, we will consider the simplest computational scheme, which belongs to the class of the local search methods. The concept of a \textit{neighborhood} in a search space is at the core of the algorithms from this class. With each point of a search space the \textit{neighborhood function} \cite{Burke} associates a set of neighboring points. This~set is called the neighborhood of the considered point. 
For an $n$-dimensional Boolean hypercube the neighborhood function is of the following kind:
\begin{equation}
\label{eq2}
    \aleph: \{0,1\}^n \rightarrow 2^{\{0,1\}^n}.
\end{equation}

A simple way to define function \eqref{eq2} is to associate an arbitrary $\alpha\in \{0,1\}^n$ with all points from $\{0,1\}^n$ for which the Hamming distance \cite{McWilliams} from $\alpha$ is not greater than certain $d$.
The number $d$ is referred to as a \textit{radius} of Hamming neighborhood. Hereinafter by $\aleph_d(\alpha)$ we denote a neighborhood of radius $d$ of 
an arbitrary point $\alpha$ of a search space.
By $\langle\{0,1\}^n,\aleph_1 \rangle$ we denote a space $\{0,1\}^n$ in which a neighboorhood of an arbitrary point $\alpha$ is $\aleph_1(\alpha)$.

Below we give a simple example of the local search algorithm which is sometimes referred to as Hill Climbing (HC). We can use this algorithm to maximize the functions of the kind \eqref{eq1}.
One iteration of the HC algorithm consists of the following steps.

\begin{enumerate}
\item[] \textbf{Input}: an arbitrary point $\alpha\in\{0,1\}^n$, a value $f(\alpha)$;
\item $\alpha$ -- current point;

\item traverse the points from  $\aleph_1(\alpha) \setminus \{\alpha\}$, computing for each point $\alpha'$ from this set a value $f(\alpha')$. If there is such a point $\alpha'$, that $f(\alpha') > f(\alpha)$ then go to step 3, otherwise, go to step 4;

\item $\alpha \leftarrow \alpha'$, $f(\alpha) \leftarrow f(\alpha')$, go to step 1;

\item $\alpha^*\leftarrow \alpha$; $(\alpha^*,f(\alpha^*))$ is a local extremum of $f$ on $\{0,1\}^n$;
\item[] \textbf{Output}: $(\alpha^*,f(\alpha^*))$.

\end{enumerate}

By itself, Hill Climbing is a basic heuristic and, generally speaking, it does not guarantee that the global extremum of the considered function will be achieved (except for some specific cases). Usually, during the optimization of an arbitrary function \eqref{eq1} one attempts to go through a number of local extrema. As a result, a point with the best value of the \textit{objective function} \eqref{eq1} is considered to be an output. The best value of this  function found at the current moment is called \textit{Best Known Value} (BKV). 

Without any exaggeration it can be said that over the past half century a huge number of papers have been devoted to describing ways of escaping local extrema.
Listing the key papers in this direction would take up too much space. A good review of the relevant results can be found in \cite{Burke,Luke}.

In some sense, one can view the evolutionary algorithms \cite{Luke} as the alternative to local search methods. This class of algorithms can be described as "a variation on a theme of random walk". The simplest example of such algorithms is the (1+1)-Evolutionary Algorithm shortly denoted as (1+1)-EA \cite{Rudolf}. Below we present the description of one iteration of this algorithm, which will be referred to as \emph{(1+1)-random mutation}.

\begin{itemize}
\item[] \textbf{Input}: an arbitrary point $\alpha\in\{0,1\}^n$, a value $f(\alpha)$;
\item make (1+1)-random mutations of $\alpha$: by going through $\alpha$ in fixed order, change every bit to the opposite with probability $p$; let $\alpha'$ be a result of a random mutation of $\alpha$;
\item 
if for a point $\alpha'$ it holds that $f(\alpha') \geq f(\alpha)$ (assuming that the maximization problem for function \eqref{eq1} is considered), then the next (1+1)-random mutation is applied to $\alpha'$, otherwise, (1+1)-random mutation is applied to $\alpha$ (this situation is called \emph{stagnation});
\item[] \textbf{Output}: $(\alpha', f(\alpha'))$, where $\alpha'$ is the result of several random mutations.
\end{itemize}

The probability $p$ is usually determined as $p = 1/n$. 
It should be noted, that for any function of the kind \eqref{eq1} and points $\alpha, \alpha' \in \{0,1\}^n$ the probability of transition $\alpha \rightarrow \alpha'$ is non-zero. 
Let $\alpha^{\#}$ be the point of the global extremum of function \eqref{eq1}. 
According to \cite{Wegener}, the expected running time of the (1+1)-EA, denoted further as $E_{(1+1)-EA}$, 
is defined as the mean of the (1+1)-random mutations needed to achieve $\alpha^{\#}$ from an arbitrary initial point $\alpha \in \{0,1\}^n$.

The value $E_{(1+1)-EA}$ can be considered as a measure of efficiency for (1+1)-EA. If the value of function \eqref{eq1} is given by the oracle, the nature of which is not taken into account, then it could be shown (see \cite{Wegener}), that $E_{(1+1)-EA} \leq n^n$. It is important that this bound is reached (with minor reservations) for explicitly specified functions \cite{Wegener}. 
On the other hand, for an equiprobable choice of points from a hypercube $\{0,1\}^n$ the expected value for the number of checked points before achieving $\alpha^{\#}$ is not greater than $2^n$.
Thus, in the worst case scenario, (1+1)-EA is extremely inefficient. However, when applied to many practical tasks (1+1)-EA can be surprisingly productive.

\section{Merging Variables Principle (MVP)}

In this section we describe a simple technique which can be applied to the problems of optimization  of arbitrary functions of the kind \eqref{eq1}, including Black-Box functions.

Consider an arbitrary function \eqref{eq1} and the problem $f\underset{\{0,1\}^n}{\rightarrow} \max$ (or $f\underset{\{0,1\}^n}{\rightarrow} \min$). Let us associate with $\{0,1\}^n$ a set of Boolean variables  $X = \{x_1, \ldots, x_n\}$ (considering $\{0,1\}^n$ as a set of all possible assignments of variables from $X$). 

Let us fix an arbitrary positive integer $r: 1 \leq r < n$ and define a new set of variables $Y = \{y_1, \ldots,  y_r\}$. Consider an arbitrary surjection $\mu:X\rightarrow Y$.
With~an arbitrary $y_j \in Y$, $j \in \{1, \ldots, r\}$  we associate a set $X_j$ of preimages of $y_j$ in the context of mapping $\mu$.
Let us link with $y_j$ a set $D_j$, which consists of $2^{|X_j|}$  different symbols of some alphabet: $D_j = \{\beta_{1}^j, \ldots, \beta_{2^{|X_j|}}^j\}$, and fix an arbitrary bijection $\omega_j:D_j \rightarrow \{0,1\}^{|X_j|}$.
Consider a set 
$$
D^{\mu} = D_1 \times \ldots \times D_r.
$$

\begin{definition}
The elements of $D_j$ are called the values of variable
$y_j$, $j \in \{1, \ldots, r\}$ and $D_j$
is called the domain of this variable. 
An arbitrary string $\beta \in D^{\mu}$ is called an assignment of variables from $Y$. Implying all notions which were introduced above we will say that merging mapping $\mu$ is defined. The elements of $Y$ are referred to as {merged variables}.
\end{definition}

Regarding the set $D^\mu$ we note that the Hamming metric is naturally defined on $D^{\mu}$ and thus $D^{\mu}$ is a metric space. 

\begin{lemma}
\label{lemma1}
An arbitrary merging mapping $\mu: X\rightarrow Y$ defines a bijective mapping
$$
\tau_{\mu}: D^{\mu} \rightarrow \{0,1\}^n.
$$
\end{lemma}
\begin{proof} Assume that for a set of Boolean variables $X$, $|X|=n$, a merging mapping $\mu$, $\mu:X \rightarrow Y$, $|Y|=r$, $1 \leq r < n$ is given. The fact that $\mu$ is surjection means that sets $X_j$, $j \in \{1,\ldots,r\}$ do not intersect, and any variable from $X$ turns out to be in some set of the kind $X_j$. Consider an arbitrary assignment $\beta \in D^{\mu}$. Let $\beta^j$ be a symbol, located in the coordinate with the number $j, j\in\{1,\ldots,r\}$ of $\beta$. 
Consider set $X_j$. Let $\alpha^j$ be a binary string associated with an element $\beta^j$ by bijection $\omega_j$. 
Let us view $\alpha^j$ as an assignment of variables from $X_j$. 
Thus, bijections $\omega_j$, $j \in \{1, \ldots, r\}$ associate all coordinates of $\beta$ with binary strings thereby setting the values of all variables from $X$. 
Consequently, an arbitrary string $\beta \in D^{\mu}$ is associated with some string $\alpha \in \{0,1\}^n$. 
Denote the resulting function by $\tau_{\mu}: D^{\mu} \rightarrow \{0,1\}^n$. Note that $Range$ $\tau_{\mu} = \{0,1\}^n$. 
If we assume that there is a vector $\alpha \in \{0,1\}^n$, which does not have a preimage in $D^{\mu}$ for a given $\tau_{\mu}$, 
then it contradicts with the properties of bijections $\omega_j$, $j \in \{1, \ldots, r\}$. Thus, $\tau_{\mu}$ is a surjection. 
Also it is easy to see, that two arbitrary different elements from $D^{\mu}$ have different images for a given $\tau_{\mu}$ (injection). Consequently, $\tau_{\mu}$ is bijection. {The Lemma 1 is proved}.
\end{proof}

\begin{definition}
Function $\tau_{\mu}$, defined in the proof of Lemma \ref{lemma1}, is called a bijection induced by a merging mapping $\mu$.
\end{definition}

\begin{example} 
Assume that $X = \{x_1,x_2,x_3,x_4,x_5\}$. Let us define the mapping $\mu: X \rightarrow Y$, $Y = \{y_1,y_2,y_3\}$ as follows:
$$
X_1 = \{x_1,x_4\}, X_2 = \{x_2\}, X_3 = \{x_3,x_5\}.
$$
The domains of variables $y_1,y_2,y_3$ are the following: $D_1 = \{\beta_1^1,\beta_2^1,\beta_3^1,\beta_4^1\}$, $D_2 = \{\beta_1^2,\beta_2^2\}$, $D_3 = \{\beta_1^3,\beta_2^3,\beta_3^3,\beta_4^3,\}$. Bijections $\omega_j$, $j \in \{1,2,3\}$ are defined as it is shown in figure \ref{fig_1}.
Thus, the mapping $\tau_{\mu}: D^{\mu}\rightarrow \{0,1\}^5$ is defined. By Lemma \ref{lemma1} it is a bijection. For example, $\tau_{\mu}(\beta_3^1,\beta_2^2,\beta_4^3) = (11101)$.
\begin{figure}[H]
    \centering
    
\begin{tabular}{ccccc}
$\omega_1$&&$\omega_2$&&$\omega_3$\\
\begin{tabular}{|c|c|}
\hline
$\beta^1_1$& $00$\\
\hline
$\beta^1_2$& $01$\\
\hline
$\beta^1_3$& $10$\\
\hline
$\beta^1_4$& $11$\\
\hline
\end{tabular}
&&
\begin{tabular}{|c|c|}
\hline
$\beta^2_1$& $0$\\
\hline
$\beta^2_2$& $1$\\
\hline
\end{tabular}
&&
\begin{tabular}{|c|c|}
\hline
$\beta^3_1$& $00$\\
\hline
$\beta^3_2$& $01$\\
\hline
$\beta^3_3$& $10$\\
\hline
$\beta^3_4$& $11$\\
\hline
\end{tabular}
\end{tabular}
    \caption{Bijections $\omega_j$, $j \in \{1,2,3\}$ which define the mapping $\tau_{\mu}: D^{\mu}\rightarrow \{0,1\}^5$}
    \label{fig_1}
\end{figure}

\end{example}

The main idea of the technique presented below consists in transitioning from the optimization problem of the original function \eqref{eq1} on $\{0,1\}^n$ to the optimization problem of specially constructed function on $D^{\mu}$ (for a given merging mapping $\mu: X \rightarrow Y$).

\begin{definition}
Consider an optimization problem for an arbitrary function \eqref{eq1}. Let $\mu: X \rightarrow Y$  be an arbitrary merging mapping. Consider the function
$$
F_{f,\mu}:D^{\mu} \rightarrow \bbbr,
$$
defined in the following way: $F_{f,\mu}(\beta) = f(\tau_{\mu}(\beta))$, in which $\tau_{\mu}$ is a bijection induced by $\mu$. Function $F_{f,\mu}$ is called $\mu$-conjugated with $f$.
\end{definition}

\begin{lemma}
$$
\mathrm{extr}_{\{0,1\}^n}f = \mathrm{extr}_{D^{\mu}}F_{f,\mu}.
$$
(here «$extr$» can be understood as $min$ or $max$).
\end{lemma}

\begin{proof}In the context of Lemma \ref{lemma1} this equality is in fact evident. Indeed, there is a bijection $\tau_{\mu}$ between $\{0,1\}^n$ and $D^{\mu}$. The value of function $F_{f,\mu}$ in an arbitrary point $\beta \in D^{\mu}$ is equal to the value of $f$ in point $\alpha = \tau_{\mu}(\beta)$. Thus, the smallest (largest) value of $F_{f,\mu}$ on $D^{\mu}$ is equal to the smallest (largest) value of $f$ on $\{0,1\}^n$. The Lemma 2 is proved.
\end{proof}

The following property gives us the exact value of the number of different merging mappings for the set $X$ of power $n$.

\begin{lemma}
\label{lemma3}
Let $f$ be an arbitrary function of the kind \eqref{eq1}. Then, the number of different merging mappings of the kind $\mu: X \rightarrow Y$ is $\sum_{r=1}^{n-1}r!\cdot S(n,r)$, where $S(\cdot,\cdot)$ -- is {a} Stirling number of the second kind.
\end{lemma}
\begin{proof}Assume that $X = \{x_1, \ldots, x_n\}$. For an arbitrary merging mapping $\mu: X \rightarrow Y$ a set $Y$ can contain $1,2,\ldots,n-1$ variables. An arbitrary merging mapping is constructed in two steps. The first step is to divide $X$ into $r$ parts (the order of the elements in each part does not matter). As a result there is a composition of sets $X_1,\ldots,X_r$. 
At the second step each set $X_j$, $j\in \{1,\ldots,r\}$ is associated with a variable from $Y = \{y_1,\ldots, y_r\}$. The number of unordered partitionings of $n$-element set into $r$ parts is $S(n,r)$ (see, for example, \cite{Stanley}). Each~unordered partitioning of $X$ into $r$ parts can be mapped to $Y$ ($|Y|=r$) in $r!$ ways. {The Lemma 3 is proved}.
\end{proof}

Let us summarize the contents of the present section. The \textit{Merging Variables Principle} (MVP) consists in the transition from the optimization of an arbitrary function $f$ of the kind \eqref{eq1} over a Boolean hypercube to the optimization problem of a function  which is $\mu$-conjugated with $f$ over metric space $D^{\mu}$. The main goal of the further sections is to demonstrate the benefits of MVP.

\section{Combining MVP with local search}

For an arbitrary function $f$ of the kind \eqref{eq1} consider a problem $f\underset{{\{0,1\}^n}}{\rightarrow} \max$. 
Assume, that $\{0,1\}^n$ is a set of all possible assignments of variables from set $X = \{x_1,\ldots,x_n\}$. 
Consider a merging mapping $\mu: X \rightarrow Y$, $Y=\{y_1,\ldots,y_r\}$, $1 \leq r<n$ and 
a metric space (with Hamming metric) $D^{\mu} = D_1 \times \ldots \times D_r$. 
Let $\tau_{\mu}:D^{\mu} \rightarrow \{0,1\}^n$ be a bijection induced by $\mu$. 
We solve the maximization problem of function $F_{f,\mu}$ on $D^{\mu}$. 
Let us define the neighborhood function over $D^{\mu}$ in the following way. 
For~an arbitrary $\beta \in D^{\mu}$ assume that
$$
\aleph_1^{\mu}(\beta) = \{\gamma \in D^{\mu}: d_H(\beta,\gamma) \leq 1\}.
$$
In other words, the neighborhood of an arbitrary point $\beta$
contains all points from $D^{\mu}$, for which the Hamming distance $d_H$ between them and $\beta$ is at most $1$. Let~us denote a metric space $D^{\mu}$ with the neighborhood structure $\aleph_1^{\mu}$ by $\langle D^{\mu}, \aleph_1^{\mu}\rangle$.

{Below} we will use a term "random merging mapping", which refers to any construction of mapping $\mu: X \rightarrow Y$ by means of a random experiment. The most natural is a scheme of random arrangements of particles in boxes \cite{Feller}. Specifically, for a fixed $r, 1\leq r<n$ assume that an arbitrary variable $y_j$, $j\in\{1,\ldots,r\}$ is associated with a box which can accommodate $n$ particles. A set $X$ is considered as a set containing $n$ particles which are randomly scattered in $r$ boxes according to the sampling without replacement.

Below we present a variant of Hill Climbing algorithm, which uses MVP (Merging Variable Hill Climbing algorithm, MVHC).

\begin{enumerate}
\item[] \textbf{Input}: an arbitrary point $\alpha\in\{0,1\}^n$, $f(\alpha)$;
\item define a random merging mapping $\mu:X\rightarrow Y$, $Y = \{y_1,\ldots,y_r\}$, $1 \leq r < n$; 

\item construct a point $\beta=\tau^{-1}_{\mu}(\alpha)$ in $\langle D^{\mu}, \aleph_1^{\mu}\rangle$, $D^{\mu} = D_1 \times \ldots \times D_r$, where $D_j$, $j \in\{1,\ldots,r\}$ are domains of $y_j$;
    
\item run HC in $\langle D^{\mu}, \aleph_1^{\mu}\rangle$ starting from point $\beta$ for an objective function $F_{f,\mu}$; let $\beta^*$ be a local maximum, achieved in one iteration of HC;
    
\item construct a point $\alpha^* = \tau_{\mu}(\beta^*)$ ($\alpha^* \in \{0,1\}^n$);
\item[] \textbf{Output}: $(\alpha^*,f(\alpha^*))$.

\end{enumerate}

\begin{theorem}
\label{theorem1}
In the context of the MVHC scheme described above let $\beta=\tau^{-1}_{\mu}(\alpha)$ be a point in $\langle D^{\mu}, \aleph_1^{\mu}\rangle$ which is not a local maximum. Then $f(\alpha^*) > f(\alpha)$, where $\alpha^* = \tau_{\mu}(\beta^*)$ and $\beta^*$ is a local maximum, achieved by HC in $\langle D^{\mu}, \aleph_1^{\mu}\rangle$ in one iteration, starting from point $\beta$.
\end{theorem}
\begin{proof}
Let $\mu,\tau_{\mu}, D^{\mu},\alpha, \alpha^*, \beta, \beta^*$ be the objects from the description of the MVHC algorithm and the theorem formulation. Since $\beta$ is not a local maximum in the space $\langle D^{\mu}, \aleph_1^{\mu}\rangle$, then $F_{f,\mu}(\beta^*) > F_{f,\mu}(\beta)$. Thus, (by the definition of function $F_{f,\mu}$) it follows that $f(\tau_{\mu}(\beta^*)) > f(\tau_{\mu}(\beta))$. Therefore, $f(\alpha^*) > f(\alpha)$. {The Theorem is proved}.
\end{proof} 

The MVHC algorithm can be used to construct an iterative computational scheme in which the 
random merging mapping is launched multiple times: in particular, the output $\alpha^*$ of an arbitrary iteration
can be used as an input for the following iteration.

Below we would like to comment on a number of features of the proposed algorithm and show the techniques that can improve the 
practical effectiveness of MVHC. The proofs for the properties described below are not shown due to their simplicity and limitations on the volume of the paper.

\begin{enumerate}[label=\alph*.]
\item 
Note that point $\alpha$ can be a local maximum of function \eqref{eq1}  in the space $\langle\{0,1\}^n, \aleph_1\rangle$, while point $\beta = \tau_{\mu}^{-1}(\alpha)$ is simultaneously not a local maximum  of function $\mu$-conjugated with \eqref{eq1} in $\langle D^{\mu}, \aleph_1^{\mu}\rangle$. This~fact makes it possible to view MVHC as a special case of Variable Neighborhood Search (VNS) metaheuristic strategy \cite{Mlad-97,Hans-01,Hans-17}. Indeed, let $\alpha$ be an arbitrary point in  $\{0,1\}^n$, $\mu: X\rightarrow Y$ be an arbitrary merging mapping and $\tau_{\mu}:D^{\mu}\rightarrow \{0,1\}^n$ be a bijection induced by $\mu$. Define the neighborhood of  $\alpha$ in $\{0,1\}^n$ as follows:
\begin{equation}
\label{eq3}
\tilde{\aleph}(\alpha) = \{\tau_{\mu}(\gamma)|\gamma \in \aleph_1^{\mu}(\tau_{\mu}^{-1}(\alpha))\},    
\end{equation}
where $\aleph_1^{\mu}(\beta)$ is the Hamming neighborhood of radius 1 for the point $\beta$ in $D^{\mu}$. Note that \eqref{eq3} defines the neighborhood function over $\{0,1\}^n$. The different merging mappings will yield different neighborhood structures in the context of \eqref{eq3}. From this point of view, the Theorem \ref{theorem1} is the variant of the main VNS principle saying that the local extremum of a function with regard to one neighborhood structure may not be a local extremum of this function with regard to a different neighborhood structure. The Lemma \ref{lemma3} says that in the context of MVHC there exist numerous ways to construct neighborhood structures even for small $n$ and $r$ (say, $n=100$ and $r=10$). 

\item 
Let $\mu:X\rightarrow Y$, $|X| = n$, $|Y| = r$ be an arbitrary random mapping. Let~$X_1,\ldots, X_r$ be the sets of preimages of variables from $Y$ with respect to $\mu$, and $|X_1| = l_1, \ldots, |X_r| = l_r$; $l_1 + \ldots+ l_r = n$. Then~for an arbitrary point $\beta \in D^{\mu}$ the following holds:
\begin{equation}
\label{eq4}
\left|\aleph_1^{\mu}(\beta)\right| = \sum_{j=1}^{r}2^{l_j}+(1-r).    
\end{equation}
This fact means that for domains of relatively large size the traversal of points from the neighborhood $\aleph_1^{\mu}(\beta)$ can be naturally performed in parallel: each domain should be processed by an individual computing process. In more detail, assume that we have $t$ independent computing processes. Consider an arbitrary $\beta \in D^{\mu}$ and let $\beta^1$ be an arbitrary point from $D^{\mu}$, which differs from $\beta$ in coordinate number $1$ while coinciding with $\beta$ in the remaining coordinates. It is clear that in total there are $2^{l_1}-1$ points of this kind. Let~us traverse such points and compute the corresponding values of function $F_{f,\mu}$ using a computing process number $1$. We~can treat  the points which differ from $\beta$ only in the second coordinate in the similar fashion, etc. For~$t<r$ once the computing process finished the current task it can take any domains which have not yet been processed. One~process should perform the supervisor function and track whether the current Best Known Value have been improved. 

\item 
Let $\mu$ be an arbitrary merging mapping and $\beta^*$ be a local extremum of $F_{f,\mu}$ in $\langle D^{\mu}, \aleph_1^{\mu}\rangle$. 
It is easy to show that in this case $\alpha^* = \tau_{\mu}(\beta^{*})$ is a local extremum of $f$ in $\langle\{0,1\}^n, \aleph_1\rangle$. 
Assume that  $\mu_k$, $k \in \{1,\ldots,K\}$ are random merging mappings and $\alpha^* \in \{0,1\}^n$ is such a local extremum that points $\beta^*_k = \tau_{\mu_k}^{-1}(\alpha^*)$ are local extrema in the spaces $D^{\mu_k}$, $k \in \{1,\ldots,K\}$ for a large enough $K$. Then~let us call the point  $\alpha^*$ \emph{strong local extremum}.

\item 
Consider an arbitrary merging mapping $\mu:X \rightarrow Y$. Let $\alpha$ be 
an arbitrary point in $\{0,1\}^n$ and $\tilde{\aleph}(\alpha)$ be the 
neighborhood of $\alpha$ defined (with respect to fixed $\mu$) in 
accordance with \eqref{eq3}. Assume that  $l^* = \max\{l_1,\ldots,l_r\}$. 
It is easy to show that for $r \geq 2$ it holds that 
$\tilde{\aleph}(\alpha) \subset \aleph_{l^*}(\alpha)$. The power 
$\tilde{\aleph}(\alpha)$ (it is expressed by the number in the right part of 
\eqref{eq4}) can be significantly smaller than the power of 
$\aleph_{l^*}(\alpha)$. For example, if $n=100$, $r=10$ then  
$l_1 = \ldots = l_{10} = 10$, $|\tilde{\aleph}(\alpha)| = 10\times 2^{10}-9 = 10231$, 
while $|\aleph_{10}(\alpha)| > 1,5 \times 10^{13}$. 

\end{enumerate}

The property \textbf{d} essentially means that the merging mapping technique 
may be useless if the algorithm reached such a local extremum $\alpha^*$,
that the closest point (Hamming distance-wise) from  $\{0,1\}^n$ with
the better objective function value is at a distance $> l^*$ from $\alpha$.
On the first glance it might seem that this fact significantly limits
the applicability of the proposed method. However, 
it is possible to describe the supplementary 
technique for MVHC which is based on the idea to store strong local
extrema and use them to direct the search process. In this context
we will use the tabu lists concept which serves as a basis of the
tabu search strategy \cite{Glover}.

So, a strong local extremum is such a local extremum in $\{0,1\}^n$, for which
it was not possible to improve BKV even after a significant number of different 
merging mappings $\mu_k$, 
$k\in\{1,\ldots,K\}$. Let us denote such a point as  
$\alpha_1^*$. 
The~goal is to move from $\alpha_1^*$
to a point with the better BKV. Since we do not employ any knowledge about function $f$, it means that such transitions should rely on heuristic arguments. The first of the arguments is to escape the neighborhood of the kind $\aleph_{l_1^*}(\alpha_1^*)$ in $\{0,1\}^n$, where $l_1^*$ is a "critical" domain size that is known from the search history.
On the other hand, due to various reasons appealing to the "locality principle" it is undesirable to move "too far" from $\alpha_1^*$. It is especially relevant if during the transition to $\alpha_1^*$ the BKV have been improved multiple times.
Thus, the simplest step is to move to an arbitrary point situated at a distance of  $l_1^*+1$ from $\alpha_1^*$. Let $\alpha_2$ be such a point. Assume that we launch MVHC from this point and $\alpha_2^*$ is the resulting strong local extremum of $f$, which is different from $\alpha_1^*$. Similar to $l_1^*$ we can define critical domain size  $l_2^*$ used during the  transition from $\alpha_2$ to $\alpha_2^*$, critical domain size  $l_3^*$ and etc.

As a result, assume that we have strong local extrema $\alpha_1^*,\ldots,\alpha_R^*$ and our goal is to construct a point $\alpha_{R+1} \in \{0,1\}^n$ to launch the $R+1$-th iteration of MVHC from it. 
Taking into account the above, we have a problem of choosing next current point $\alpha_{R+1}$ as a point which satisfies a system of constraints of the following kind:
\begin{equation}
\label{eq6}
(d_H(\alpha_{R+1},\alpha_1^*) = L_1) \wedge    
\ldots \wedge
(d_H(\alpha_{R+1},\alpha_{R}^*) = L_{R}).
\end{equation}

The numbers $L_1,\ldots,L_{R}$ can be chosen according to different criteria.
Let us describe the simplest one. 
Consider the following system of constraints:
\begin{equation}
\label{eq7}
(d_H(\alpha_{R+1},\alpha_1^*) = l_1^*+1) \wedge \ldots \wedge (d_H(\alpha_{R+1},\alpha_{R}^*) = l_{R}^*+1).
\end{equation}
If there exists a point $\alpha_{R+1}$ that satisfies \eqref{eq7} then it is chosen 
as a starting point for the next MVHC iteration. If such a point does not exist, 
then we call \eqref{eq7} incompatible. 
In this case it is possible to relax
some of the constraints of the kind $d_H(\alpha_{R+1},\alpha_q^{*})=l_q^{*}+1$
by replacing them with constraints of the kind $d_H(\alpha_{R+1},\alpha_q^{*})=L_q$, where $L_q\geq l_q^{*}+2, q\in \{1,\ldots,R\}$. The resulting system of constraints of the
kind \eqref{eq6} is again to be tested for compatibility.

Let us consider the problem of testing the compatibility of an arbitrary system of the kind \eqref{eq6}. Consider an arbitrary constraint of the kind
$d_H(\alpha_{R+1},\alpha^{*}) = L$,
where $\alpha^{*} = (\alpha^1, \ldots, \alpha^n)$  is a known Boolean vector and $L$ is a known natural number.  Let us represent the unknown components of vector $\alpha_{R+1}$ using Boolean variables $z_1, \ldots, z_n$. Now consider the expression
\begin{equation}
\label{eq8}
(z_1 \oplus \alpha^1) + \ldots + (z_n \oplus \alpha^n) = L,
\end{equation}
where $\oplus$ is the sum $\bmod{2}$, and $+$ is an integer sum.

We can consider \eqref{eq8}
as an equation for unknown variables $z_1, \ldots, z_n$. It~is~easy to see that a set of vectors $\alpha_{R+1}$, which satisfy the constraint $d_H(\alpha_{R+1},\alpha^*) = L$, coincides with the set of solutions of the equation \eqref{eq8}. To solve the systems of equations of the kind \eqref{eq8} or to prove the inconsistent of such systems
we can use any complete algorithm for solving SAT. The corresponding reduction to SAT is {performed  effectively} using the procedures described, for example, in \cite{Een2006}. 

Thus, to choose new current points in the context of MVHC we can employ a strategy in which SAT oracles are {combined with} the tabu lists containing strong local extrema.

\section{Combining MVP with evolutionary computations}

Now let us consider how MVP can be combined with evolutionary algorithms. In
particular, let us study the MV-variant of (1+1)-EA. As it was stated above, for an arbitrary function of the kind \eqref{eq1} in \cite{Wegener} there was obtained the following upper bound: $E_{(1+1)-EA} \leq n^n$. Also in \cite{Wegener} there was given an example of a function (the \textsc{Trap} function) for which this bound is asymptotically achieved (in terms of \cite{Wegener}). 

In the description of the 
MV-variant of (1+1)-EA (we denote the corresponding algorithm as (1+1)-MVEA)
we want to preserve the following property of the original algorithm: that the 
expected value of the number of bits in which the Boolean vector is different from 
its (1+1)-random mutation should be 1.

Assume that there is an arbitrary merging mapping $\mu:X\rightarrow Y$, $|X| =n$, $|Y| = r$, $1 \leq r < n$. For an arbitrary point $\alpha \in \{0,1\}^n$ perform the following {steps}.

\begin{enumerate}
    \item[] \textbf{Input}: arbitrary point $\alpha \in \{0,1\}^n, f(\alpha)$;

    \item  construct a point $\beta = \tau^{-1}_{\mu}(\alpha)$; 
    perform $r$ Bernoulli trials with probability of success $p=1/r$;
    let $\{i_1,\ldots,i_q\} \subseteq \{1,\ldots,r\}$ be the numbers of successful
    trials; for each  $j \in \{i_1,\ldots,i_q\}$ consider the domain $D_j$ of
    a variable $y_j$, let $X_j$ be the set of preimages of $y_j$ for the mapping
    $\mu$, $\omega_j:D_j \rightarrow \{0,1\}^{|X_j|}$ is a fixed bijection,
    $\beta_j$ is the value of $y_j$ in $\beta$; 
    
    \item consider the Boolean vector $\alpha_j = \omega_j(\beta_j)$ of size $l_j = |X_j|$; perform (1+1)-random mutation on $\alpha_j$ with probability of success equal to 
    $\frac{1}{l_j}$, let $\alpha_j'$ be the result of the mutation, $\beta_j' = \omega^{-1}_j(\alpha_j')$; 
    
    \item construct a point $\beta'$ in $D^{\mu}$: in the coordinate with number
    $j \in \{i_1,\ldots,i_q\}$ the point $\beta'$ has $\beta_j'$; in the remaining
    coordinates with numbers from the set $\{1,\ldots,r\}\setminus
    \{i_1,\ldots,i_q\}$ the point $\beta'$ coincides with $\beta$;
    
    \item construct a point $\alpha' = \tau_{\mu}(\beta')$ ($\alpha' \in \{0,1\}^n$);
    \item[] \textbf{Output}: $(\alpha',f(\alpha'))$.
\end{enumerate}

\begin{definition}
To the described sequence of actions the result of which is the transition $\alpha \rightarrow \alpha'$ we will refer as (1+1)-merging variable random mutation. 
\end{definition}

\begin{lemma}
\label{lemma4}
For an arbitrary merging mapping $\mu$ the expected value of the number of bits in which 
the points $\alpha$ and $\alpha'$ differ is 1.
\end{lemma}
\begin{proof}
Let $\mu$ be an arbitrary merging mapping and $\alpha$ be an arbitrary vector from $\{0,1\}^n$. Let us apply to $\alpha$ (1+1)-merging variable random mutation with respect to $\mu$. Consider two kinds of random variables. The random variables of the first kind are the independent Bernoulli variables denoted as $\zeta^j$, $j \in \{1, \ldots r\}$. For each $j \in \{1, \ldots r\}$ the variable $\zeta^j$ has spectrum $\{0,1\}$ and distribution $\{1 - \frac{1}{r}, \frac{1}{r}\}$. If $\zeta^j = 1$ then we apply to $a_j$ an original (1+1)-random mutation with success probability $\frac{1}{l_j}$, where $l_j = |X_j|$.
As above, here we mean that $\alpha_j$ is an assignment of variables from $X_j$.

Also consider the random variables $\xi^j$: the value of each such variable is equal to the number of bits changed in vector $\alpha_j$, $j \in \{1, \ldots r\}$, as a result of applying to $\alpha_j$ (1+1)-random mutation with success probability $\frac{1}{l_j}$ ($l_j = |X_j|$). Thus, for each $j \in \{1, \ldots r\}$ the variable $\xi^j$ takes value from the set $\{0,1,\ldots,l_j\}$.
Then the number of changed bits in vector $\alpha_j$ after (1+1)-merging variable random mutation is a random variable
\begin{equation}
\label{T}
\theta = \sum^{r}_{j=1}\zeta^j \cdot \xi^j.
\end{equation}
Note that for each $j \in \{1, \ldots r\}$ the variables $\zeta^j$ and $\xi^j$ are obviously independent. Then from \eqref{T} the following holds: $$E[\theta] = \sum^{r}_{j=1}E[\zeta^j \cdot \xi^j] = \sum^{r}_{j=1}E[\zeta^j] \cdot E[\xi^j] = r \cdot \frac{1}{r} \cdot 1 = 1.$$
Thus, the Lemma 4 is proved.
\end{proof}

\begin{definition}
For a fixed merging mapping $\mu$, the (1+1)-merging variable evolutionary algorithm ((1+1)-MVEA) is a sequence of (1+1)-merging variable random mutations.
In the context of maximization problem of an arbitrary function \eqref{eq1}: the next mutation is applied to $\alpha'$ if $f(\alpha') \geq f(\alpha)$. Otherwise, the next mutation is applied to $\alpha$ (stagnation).
\end{definition}
The following definition is a variant of the Definition 5 from \cite{Wegener} with relation to (1+1)-MVEA.
\begin{definition}
Let $f$ be an arbitrary function of the kind \eqref{eq1} and $\alpha^{\#}$ be a global extremum of function $f$ on $\{0,1\}^n$. Let $\mu$ be an arbitrary merging mapping. We will define the expected running time of (1+1)-MVEA as the mean of the number of (1+1)-merging variable random mutations that have to be applied to an arbitrary point $\alpha \in\{0,1\}^n$ until it transforms into $\alpha^{\#}$. Denote this value by $E^{\mu}_{(1+1)-MVEA}$. 
\end{definition}

\begin{theorem}
\label{theorem2}
Assume that $f$ is an arbitrary function of the kind \eqref{eq1}, 
$\mu:X\rightarrow Y$ is an arbitrary merging mapping: 
$X = \{x_1,\ldots,x_n\}$, $Y = \{y_1,\ldots,y_r\}$, $1 \leq r<n$, $l_j = |X_j| \geq 2$ 
for all  $j \in \{1,\ldots,r\}$ and $l = 
\max\{l_1,\ldots,l_r\}$. 
Then the following estimation holds:
\begin{equation}
\label{eq9}
E^{\mu}_{(1+1)-MVEA} \leq r^r \cdot l^n.
\end{equation}
\end{theorem}
\begin{proof}
Let $\mu$ be an arbitrary merging mapping for which all the conditions of the
theorem are satisfied. Now let us reason in a way similar to the proof of the
Theorem 6 in \cite{Wegener}. Let $\alpha \in \{0,1\}^n$ be an arbitrary point and
$\alpha^{\#}$ be a global extremum of the function \eqref{eq1} on $\{0,1\}^n$.
Denote by  $P_{\alpha \rightarrow \alpha^{\#}}$ the probability that $\alpha$
will transition into $\alpha^{\#}$ as a result of one iteration of the
(1+1)-MVEA-algorithm. Consider the points $\beta = \tau_{\mu}^{-1}(\alpha)$,
$\beta^{\#} = \tau_{\mu}^{-1}(\alpha^{\#})$ from the space $D^{\mu}$. 
In this context, for an arbitrary $j \in \{1,\ldots,r\}$ with the coordinates
$\beta_j$, $\beta_j^{\#}$ there will be associated the binary strings 
$\alpha_j$, $\alpha_j^{\#}$.

Now let us construct the lower bound for the probability of an {event} 
that as a result of one (1+1)-MVEA iteration there will take place a transition
from $\alpha$ to $\alpha^{\#}$. It is clear that this may happen if and only
if there takes place the transition from $\beta$ to $\beta^{\#}$.
Let $q = d_{H}(\beta,\beta^{\#})$ be the Hamming distance between 
$\beta$ and $\beta^{\#}$ in the space $D^{\mu}$. Assume that the set 
$J = \{i_1,\ldots,i_q\}\subseteq \{1,\ldots,r\}$ contains the numbers of
coordinates in $\beta$, in which this point differs from $\beta^{\#}$, and 
$U = \{1,\ldots,r\}\setminus J$. Let us denote by 
$\sigma = (\sigma_1,\ldots,\sigma_r)$, $\sigma_i \in \{0,1\}$, $i \in
\{1,\ldots,r\}$ the set of results of a sequence of $r$ Bernoulli trials
with success probability $1/r$ (as usually, we assume that $\sigma_1 = 1$
corresponds to success).

The transition $\beta \rightarrow \beta^{\#}$ takes place if and only if 
within one (1+1)-merging variable random mutation the following two events 
denoted by  $A_j$ and $B_u$ happen simultaneously:
\begin{enumerate}[label=\alph*.]
    \item for an arbitrary $j \in J$ the event $A_j$ takes place if and
    only if $\beta_j \rightarrow \beta_j^{\#}$;
    
    \item for an arbitrary $u \in U$ the event $B_u$ takes place if and
    only if $\beta_u \rightarrow \beta_u$.
\end{enumerate}
It is easy to see that all the events of the kind $A_j, B_u, j \in J, u \in U$
are independent, thus
$$
P_{\alpha \rightarrow \alpha^{\#}} = \left(\prod_{j\in J}\Pr\{A_j\}\right)\cdot\left(\prod_{u \in U}\Pr\{B_u\}\right).
$$
For an arbitrary $k \in \{1,\ldots,r\}$ let us denote by $p_k$ 
the probability that the result of the random (1+1)-mutation 
with probability of success $\frac{1}{l_k}$ of the string 
$\alpha_k = \omega_k(\beta_k)$ is the string $\alpha_k^{\#}$. 
Then for any $j \in J$ it holds that 
$\Pr\{A_j\} = \frac{1}{r}\cdot p_j$. 

For an arbitrary $u \in U$ the event $B_u$ can happen in
one of the two cases: first if $\sigma_u = 0$,  and, second, if 
$\sigma_u = 1$, but the result of the (1+1)-random mutation 
with the probability of success $\frac{1}{l_u}$ of the string $\alpha_u = \omega_u(\beta_u)$ is the string $\alpha_u$. 
In~the first case, $\Pr\{B_u\} = (1-\frac{1}{r})$. In the second case,
$\Pr\{B_u\} = \frac{1}{r}\cdot(1 - \frac{1}{l_u})^{l_u}$. Thus,
in any case when $r \geq 2$, $l_u \geq 2$ it holds that 
$\Pr\{B_u\} \geq \frac{1}{r}\cdot \frac{1}{l_u^{l_u}}$. 
Taking this fact into account the following bound holds:
\begin{equation}
\label{eq10}
P_{\alpha \rightarrow \alpha^{\#}} \geq \left(\frac{1}{r^q}\cdot\prod_{j\in J}p_j\right)\cdot\left(\frac{1}{r^{r-q}}\cdot\prod_{u \in U}\frac{1}{l_u^{l_u}}\right).
\end{equation}

In accordance with \cite{Wegener} for an arbitrary  $k \in \{1,\ldots,r\}$, 
such that $l_k \geq 2$, the following holds: $p_k \geq \frac{1}{l_k^{l_k}}$.
Together with \eqref{eq10} this fact gives us the next bound:
\begin{equation}
\label{eq11}
P_{\alpha \rightarrow \alpha^{\#}} \geq \frac{1}{r^{r}}  \cdot \prod_{k \in \{1,\ldots,r\}} \frac{1}{l_k^{l_k}}.
\end{equation}
Let us emphasize that \eqref{eq11} holds for an arbitrary $\alpha \in \{0,1\}^n$.
Assume that $l = \max\{l_1,\ldots,l_r\}$. Then, taking into account that 
$\sum_{k=1}^r l_k = n$, it follows from~\eqref{eq11}:
$$
P_{\alpha \rightarrow \alpha^{\#}} \geq \frac{1}{r^{r}}  \cdot \frac{1}{l^n}.
$$
The bound \eqref{eq9} follows from the latter inequality. The Theorem 2 is thus proved.
\end{proof}

The bound \eqref{eq9} looks a little surprising since it is actually easy to
determine the merging mappings with such parameters $r$ and $l$ that the 
corresponding variant of the bound \eqref{eq9} becomes significantly better
than the similar bound for (1+1)-EA shown in \cite{Wegener}. 

\begin{definition}
Assume that $|X| = n$, $|Y| = r$, $1 \leq r<n$ and $n = \lfloor\frac{n}{r}\rfloor\cdot r +
b$, where $b, b \in \{0, \ldots,r-1\}$ is the remainder from the division of $n$ 
by $r$. Let $\mu: X\rightarrow Y$ be an arbitrary merging mapping, such that 
for $b$ sets of the kind $X_j$, $j \in \{1,\ldots,r\}$ it holds that $|X_j| =
\lfloor\frac{n}{r}\rfloor+1$, and for the remaining $r-b$ sets of such kind $|X_j| 
= \lfloor\frac{n}{r}\rfloor$. Let us refer to such $\mu$ as  {uniform merging
mapping}.
\end{definition}

\begin{corollary}
Let $\mu: X\rightarrow Y$ be an arbitrary uniform merging mapping such that
$l_j \geq 2$ for all $j \in \{1,\ldots,r\}$. Then there exists such a function
$\delta(n):1<\delta(n) \leq n$, that the following evaluation holds:
\begin{equation}
\label{eq12}
E^{\mu}_{(1+1)-MVEA} \leq n^{n \cdot \left(\frac{1}{\delta(n)} - \frac{\log_n\delta(n)}{\delta(n)}+\log_n(\delta(n)+1)\right)}.
\end{equation}
\end{corollary}
\begin{proof}
Let $\mu: X\rightarrow Y$ be an arbitrary uniform merging mapping. 
By definition it means that $2 \leq l \leq \frac{n}{r}+1$ for all $j \in \{1, \ldots, r\}$ and, thus we can use the evaluation (9):
\begin{equation}
\label{eq13}
    E^{\mu}_{(1+1)-MVEA} \leq r^r \cdot {\left(\frac{n}{r} + 1\right)}^n.
\end{equation}
Now introduce  $\delta(n):\delta(n) = n/r$. Then $1<\delta(n) \leq n$.
Taking this into account we can transform \eqref{eq13} as follows:
\begin{multline*}
 E^{\mu}_{(1+1)-MVEA} \leq \left(\frac{n}{\delta(n)}\right)^{\frac{n}{\delta(n)}}
 \cdot(\delta(n)+1)^n = n^{\frac{n}{\delta(n)}}\cdot 
 \left(\delta(n)\right)^{-\frac{n}{\delta(n)}}\cdot (\delta(n)+1)^n=\\
 = n^{\frac{n}{\delta(n)}}\cdot n^{-\frac{n}{\delta(n)}\cdot 
 \log_n\delta(n)}\cdot
 n^{n\cdot \log_n (\delta(n)+1)}= 
 n^{n\cdot \left(\frac{1}{\delta(n)} -\frac{\log_n\delta(n)}{\delta(n)} + \log_n(\delta(n)+1)\right)
 }.
\end{multline*}
Thus the Corollary 1 is proved.
\end{proof}

Based on \eqref{eq12} it is possible to give a number of examples of uniform 
merging mappings, that provide better worst-case-estimations of (1+1)-MVEA 
for an arbitrary function of the kind \eqref{eq1} compared to the similar estimation
for (1+1)-EA from \cite{Wegener}. Indeed, for example for
$\delta(n) \sim \sqrt[3]{n}$ and for any  $n \geq 27$ it follows from \eqref{eq12} 
that $E^{\mu}_{(1+1)-MVEA} \lesssim
n^{n\cdot(\frac{1}{\sqrt[3]{n}}-\frac{1}{3\sqrt[3]{n}}+\frac{1}{2})}$
(here it is taken into account that for $n \geq 27$ it holds that 
$\log_n(\sqrt[3]{n} + 1) < \frac{1}{2}$). Thus in this case the following holds 
$E^{\mu}_{(1+1)-MVEA} \lesssim n^{(\frac{n}{2}+\frac{2}{3}n^{2/3})}$.

\section{Preliminary computational results}
The MVHC was implemented in the form of a multi-threaded C++ application. It 
employs the parallel variant of the procedure for traversing the neighborhoods
in the search space (see Section 3).

In the role of test instances we considered the problems of finding preimages of some cryptographic functions reduced to the Boolean Satisfiability problem (SAT). 
Such instances are justified to be hard, thus they 
can be viewed as a good test suite to compare the effectiveness of combinatorial
algorithms. At the current stage we considered the problems of finding
preimages of a well-known MD4 cryptographic hash function \cite{Rivest} with
additional constraints on the hash value. In particular, the goal was to find such 512-bit inputs that yield MD4 hash values with leading zeros. This problem can be 
reduced to SAT effectively. For this purpose we employed the Transalg software system \cite{Transalg}.

Let $\{0,1\}^{512} \rightarrow \{0,1\}^{128}$
be a function which is defined by the MD4 algorithm. 
Let $C$ be a CNF which encodes this algorithm.
In the set of variables from $C$ let us select two sets.
First set is $X^{in}$, which consists of 512 Boolean variables encoding an input of MD4.
Second one is $X^{out}$ -- a set of 128 Boolean variables encoding the output of MD4.
In the set $X^{out}$ select $k$ variables encoding the leading bits of the hash value, and assign these variables with value 0. Denote the resulting CNF as $C_k$. This CNF is satisfiable and from any satisfying assignment one can effectively extract such $\alpha \in \{0,1\}^{512}$ for which the leading $k$ bits of corresponding MD4 hash value are equal to zero.

To find the satisfying assignment for $C_k$ we used two approaches. 
First we applied to $C_k$ the multithreaded solvers, based on the CDCL algorithm \cite{Joao-Handbook}, that won the yearly SAT competitions in recent years.
In the second approach we used the MVHC algorithm described in the Section 3 of the present paper. 
Consider,~a~set~of variables $X^{in}, |X^{in}| = 512$ in CNF $C_k$. Associate an arbitrary vector $\alpha \in \{0,1\}^{512}$ with a set of literals over variables from $X^{in}$.
Recall, that a literal is either the variable itself or its negation.
If a component of vector $\alpha$ corresponding to a variable $x_i$, $i \in \{1,\ldots,512\}$ takes value 1, then the corresponding literal is $x_i$. Otherwise, the literal is $\lnot x_i$. All such literals are conjunctively added to CNF $C_k$ and the  resulting CNF is denoted by $C_k(\alpha)$.
It is well known that set $X^{in}$ is a Strong Unit Propagation Backdoor Set (SUPBS) for CNF $C_k$ \cite{Williams}. This~means that the satisfiability of CNF $C_k(\alpha)$ can be checked in time linear on the size of this CNF using a simple procedure of Boolean constraints propagation called Unit Propagation Rule \cite{Joao-Handbook}.
Thus, we consider function of the kind \eqref{eq1} which associates with an arbitrary $\alpha \in \{0,1\}^{512}$ a number of clauses in $C_k(\alpha)$ that take the value of $1$ as a result of application of Unit Propagation rule to CNF $C_k(\alpha)$. 
If the value of this function is equal to the number of clauses that are satisfied in $C_k(\alpha)$, then $\alpha$ is 
a MD4 preimage of a hash value with $k$ leading zero bits.
For this function the problem of maximization on $\{0,1\}^{512}$ was solved using MVHC algorithm, in which uniform merging mapping was employed.

All tested {algorithms} were run on a personal computer (Intel Core i7, 16 GB RAM) in 8 threads. Since these algorithms are randomized, the result of each test is an average time of three independent launches for each algorithm. The~obtained results are presented in Table~\ref{results}.

\begin{table}[ht]
\centering
\caption{An average time (in seconds) of finding a MD4 preimage for hash value with $k$ leading zero bits. For MVHC algorithm an uniform merging mapping was used}
\label{results}
\begin{tabular}{|l|l|l|l|}
\hline
Solver & $k = 18$ & $k = 20$ & $k = 22$\\
\hline
\textsc{MVHC} ($l$=4) & 244.8 & 1028 & 2126\\\hline
\textsc{MVHC} ($l$=8) & 490.1 & 1044.8  & 2003.1\\ \hline
\textsc{MVHC} ($l$=12) & 30 & 105.9 & 1882.8\\\hline
\textsc{cryptominisat} \cite{Soos} &429.1 & 1197.9 & 3197.5\\\hline
\textsc{plingeling} \cite{Biere} & 2175.1 & 1840.3 & 4218.4\\ \hline

\end{tabular}
\end{table}

\section{Related Work (briefly)}

As it was mentioned above, there is a large set of metaheuristics and corresponding discussion contained in the monograph \cite{Luke} by S.~Luke.
One of the first papers in which some complexity estimations of the simplest evolutionary algorithm (1+1)-EA were presented was G.~Rudolf's {dissertation}  \cite{Rudolf}.

Variable Neighborhood Search method (VNS) was first proposed in \cite{Mlad-97} and developed in subsequent papers: \cite{Hans-01,Hans-17} and a number of others.
Also we would like to note that the ideas underlying the MVP are similar in nature to those previously used in papers dedicated to the application of  Large Scale Neighborhood Search \cite{Ahuja,Igor}. 

A number of results on the complexity estimation of evolutionary algorithms originates in \cite{Wegener}. These studies are actively conducted to the present day. From~the latest results in this area one should note \cite{Doerr}.

We emphasize that {MaxSAT} is not the main object of study of the present paper. The~special case of MaxSAT, related to the preimage finding problem of cryptographic functions, was considered only as an example of the maximization problem of pseudo-Boolean function. Listing  the key papers devoted to SAT and MaxSAT would take up too much space.
In this context, we refer only to the well-known handbook \cite{Handbook} and, in particular, to its chapter on MaxSAT \cite{MAXSAT}. It~should be noted that in a number of papers various metaheurists were used to solve MaxSAT, employing both local search (see \cite{Ansoteg,Bouhmalla}, etc.) and the concept of evolutionary computations (see, for example, \cite{Bouhmalla,Buzdalov-Doerr}).

\section{Conclusion and Acknowledgements}
In the present paper we described a metaheuristic technique focused on the problem of pseudo-Boolean optimization.
Arguments were given for using this technique both in combination with local search methods and in conjunction with evolutionary algorithms.
The proposed technique when applied to local search methods can be considered as a special case of Variable Neighborhood Search.
The first program implementation of the technique turned out to be quite effective in application to some reasonably hard problems of pseudo-Boolean optimization.

The author expresses deep gratitude to Ilya Otpuschennikov for the program implementation of MVHC algorithm. The author also thanks Maxim Buzdalov for productive discussion and useful advice.

The research was funded by Russian Science Foundation (project No. 16-11-10046).

\end{document}